\documentclass{article}
\usepackage{spconf}

\usepackage[utf8]{inputenc} 
\usepackage[T1]{fontenc}    
\usepackage{hyperref}       
\usepackage{url}            
\usepackage{booktabs}       
\usepackage{amsfonts}       
\usepackage{nicefrac}       
\usepackage{microtype}      
\usepackage{flushend}
\usepackage{algorithm}
\usepackage{algorithmic}
\usepackage{amsmath}
\usepackage{amssymb}
\usepackage{amsthm}
\usepackage{bbm}
\usepackage{color}
\usepackage{enumitem}
\usepackage{fancyhdr}
\usepackage{graphicx}
\usepackage{listings}
\usepackage{mathtools}
\usepackage{setspace}
\usepackage{flushend}
\newtheorem{thm}{Theorem}
\newtheorem{lemma}{Lemma}

\newtheorem{assumption}{Assumption}

\theoremstyle{definition}

\usepackage{mathtools, cuted}
\newcommand{\field}[1]{\mathbb{#1}}

\newcommand{\R}{\field{R}}

\newcommand{\E}{\field{E}}

\newcommand{\bg}{\boldsymbol{g}}
\newcommand{\bG}{\boldsymbol{G}}
\newcommand{\bx}{\boldsymbol{x}}

\newcommand{\bI}{\boldsymbol{I}}
\newcommand{\bw}{\boldsymbol{w}}

\newcommand{\bu}{\boldsymbol{u}}
\newcommand{\bU}{\boldsymbol{U}}
\newcommand{\bv}{\boldsymbol{v}}

\title{No-regret Non-convex Online Meta-Learning}
\name{Zhenxun Zhuang $^{1}$, Yunlong Wang $^2$,Kezi Yu $^2$, Songtao Lu $^3$\thanks{The work was done when S.L. was with the University of Minnesota.}}
\address{$^1$ Department of Computer Science, Boston University, Boston, MA 02215 \\  $^2$ Advanced Analytics, IQVIA, Plymouth Meeting, PA 19462 \\  $^3$ IBM Research AI,  IBM Thomas J. Waston Research Center, NY 10598} 
\begin{document}
\ninept
\maketitle
\begin{abstract}
The online meta-learning framework is designed for the continual lifelong learning setting. It bridges two fields: meta-learning which tries to extract prior knowledge from past tasks for fast learning of future tasks, and online-learning which deals with the sequential setting where problems are revealed one by one. In this paper, we generalize the original framework from convex to non-convex setting, and introduce the local regret as the alternative performance measure. We then apply this framework to stochastic settings, and show theoretically that it enjoys a logarithmic local regret, and is robust to any hyperparameter initialization. The empirical test on a real-world task demonstrates its superiority compared with traditional methods.
\end{abstract}

\section{Introduction}
\label{sec:intro}

In recent years, high-capacity machine learning models, such as deep neural networks~\cite{lecun2015deep}, have achieved remarkable successes in various domains~\cite{silver2016mastering,redmon2016you,amodei2016deep}. However, domains where data is scarce remain a big challenge as those models' ability to learn and generalize relies heavily on the abundance of training data. In contrast, humans can learn new skills and concepts very efficiently from just a few experiences. This is because when encountering a new task, learning algorithms start completely from scratch; while humans are typically armed with plenty of prior knowledge accumulated from past experience which may share overlapping structures with the current task, and thus can enable efficient learning of the new task.

Meta-learning~\cite{naik1992meta, thrun2012learning, vinyals2016matching} was designed to mimic this human ability. A meta-learning algorithm is first given a set of meta-training tasks assumed to be drawn from some distribution, and attempts to extract prior knowledge applicable to all tasks in the form of a meta-learner. This meta-learner is then evaluated on an unseen task, usually assumed to be drawn from a similar distribution as the one for training. Recent years have seen a surge of interests in this field resulting in numerous achievements, among which a seminal work is the gradient-based algorithm: MAML~\cite{finn2017model}. Due to its simplicity yet great efficiency and generality, it has initiated a fruitful line of research~\cite{nichol2018first, antoniou2018how, collins2020distribution}. However, like other meta-learning algorithms, it assumes all meta-training tasks are available together as a batch, which doesn't capture the sequential setting of continual lifelong learning in which new tasks are revealed one after another.

Meanwhile, online learning~\cite{cesa2006prediction}
specifically tackles the sequential setting. At each round $t$, one picks an $\bx_t$, and suffers a loss $f_t(\bx_t)$ revealed by a potentially adversarial environment. The goal is to minimize the \emph{regret}, the difference between the cumulative losses suffered by the algorithm and that of any fixed predictor, formally:
\begin{equation}
\text{Regret}_T(\bx)
:=
\sum^T_{t=1}f_t(\bx_t) - \sum^T_{t=1}f_t(\bx)~.
\label{eq:reg}
\end{equation}
Yet, online learning sees the whole process as a single task without adaptation for each single step.

Neither paradigm alone is ideal for the continual lifelong learning scenario, thus, Finn et al.~\cite{finn2019online} proposed to combine them together to construct the Online Meta-Learning framework which will be discussed in Section~\ref{sec:background}. However, this framework has a strong convexity assumption, while many problems of current interest have a non-convex nature. Thus, in Section~\ref{sec:noncon}, we generalize this framework to the non-convex setting. Section~\ref{sec:algo} presents an exemplification of our algorithm with rigorous theoretical proofs of its performance guarantee. Real data experiment results are shown in Section~\ref{sec:exp}. In the end, concluding remarks and takeaways are provided in Section~\ref{sec:conc}. To the best of our knowledge, it is the first theoretical regret analysis for non-convex online meta-learning algorithms, shedding the light of applying online meta-learning for more challenging learning problems in the paradigm of deep neural networks.





\textbf{Notation.}\ We use bold letters to denote vectors, e.g., $\bu,\bG\in\R^d$. The $i$th coordinate of a vector $\bu$ is $u_i$. Unless explicitly noted, we study the Euclidean space $\R^d$ with the inner product $\langle\cdot, \cdot\rangle$, and the Euclidean norm. We assume everywhere our objective function $f$ is bounded from below and denote the infimum by $f^\star > -\infty$. The gradient of a function $f$ at $\bx$ is $\nabla f(\bx)$. $\E[\bu]$ means the expectation w.r.t. the underlying probability distribution of a random variable $\bu$.

\section{Background}
\label{sec:background}

\begin{algorithm}[t]
\begin{algorithmic}[1]
\STATE{\textbf{Input}: An initial meta-learner $\bw_1$, a loss function $\ell(\cdot)$, a local adapter $\bU(\cdot)$, an online learning algorithm $\mathcal{A}$}
\FOR{$t = 1,2,\ldots$}
\STATE{Encounter a new task: $\mathcal{T}_t$}
\STATE{Receive training data for current task: $\mathcal{D}_t^{tr}$}
\STATE{Adapt $\bw_t$ to current task: $\hat{\bw}_t = \bU(\bw_t, \mathcal{D}_t^{tr})$}
\STATE{Receive test data for current task: $\mathcal{D}_t^{ts}$}
\STATE{Suffer $\ell_t(\bw_t)\triangleq\ell(\hat{\bw}_t, \mathcal{D}^{ts}_t)=\E_{\bx,y\sim\mathcal{D}^{ts}_t}[\ell(\hat{\bw}_t, \bx;y)]$}
\STATE{Update $\bw_{t+1} = \mathcal{A}(\bw_1,\ell_1(\bw_1),\ldots,\ell_t(\bw_t))$}
\ENDFOR
\end{algorithmic}
\caption{Online Meta-Learning}
\label{algo:mlol}
\end{algorithm}

Algorithm~\ref{algo:mlol} is the online meta-learning framework proposed in~\cite{finn2019online}. A meta-learner $\bw_t$ is maintained to preserve the prior knowledge learned from past rounds. For each new task $\mathcal{T}_t$, one is first given some training data $\mathcal{D}_t^{tr}$ for adapting $\bw_t$ to the current task following some strategy $\bU(\cdot)$. Then the test data $\mathcal{D}_t^{ts}$ will be revealed for evaluating the performance of the adapted learner $\hat{\bw}_t$. The loss suffered at this round $\ell_t(\bw_t)$ can then be fed into an online learning algorithm $\mathcal{A}$ to update $\bw_t$. We use $\bU(\bw_t,\mathcal{D}^{tr}_t)=\bw_t - \alpha\nabla \E_{\bx,y\sim\mathcal{D}^{tr}_t}[\ell(\bw_t, \bx;y)]$ following~\cite{finn2019online} where $\alpha$ is the step-size.

As tasks can be very different, the original regret in Equation~\eqref{eq:reg} of competing with a fixed learner across all tasks becomes less meaningful. Thus, Finn et al.~\cite{finn2019online} changed it to:
\begin{equation*}
\text{Regret}'_T(\bw) =
\sum^T_{t=1}\ell(\bU(\bw_t, \mathcal{D}_t^{tr}), \mathcal{D}_t^{ts}) - \sum^T_{t=1}\ell(\bU(\bw, \mathcal{D}_t^{tr}), \mathcal{D}_t^{ts})~,
\label{eq:reg2}
\end{equation*}
which competes with any fixed \emph{meta-learner}. Under this, they designed the Follow the Meta Leader algorithm enjoying a logarithmic regret when assuming strong-convexity on $\ell$.

\section{Problem Formulation}
\label{sec:noncon}
In this section, we generalize the online meta-learning algorithm to non-convex setting by first demonstrating the infeasibility of regret of form~\eqref{eq:reg} and then introducing an alternative performance measure.

Finding the global minimum for a non-convex function in general is known to be NP-hard. Yet, if we could find an online learning algorithm with a $o(T)$ regret for some non-convex function classes, we can optimize any function $f$ of that class efficiently: simply run the online learning algorithm but with the objective $f$ as the loss $\ell_t$ at each round, and choose a random update as output. This gives us:
\begin{equation*}
\begin{aligned}
& \E_i[f(\bw_i)]-\min_{\bw\in\mathcal{K}} f(\bw)
=
\frac1T\sum^T_{t=1} f(\bw_t) - \min_{\bw\in\mathcal{K}} f(\bw) \\
& \phantom{oghaooo}= 
\frac1T\sum^T_{t=1} \ell_t(\bw_t) - \min_{\bw\in\mathcal{K}} \frac1T\sum^T_{t=1} \ell_t(\bw)
\in
o(1)~,
\end{aligned}
\end{equation*}
which leads to a contradiction unless P=NP. Thus, we have to find another performance measure for the non-convex case. One potential candidate is the local regret proposed by Hazan et al.~\cite{hazan2017efficient}:
\begin{equation}
\begin{aligned}
& \mathcal{R}_{m}(T)
\triangleq
\sum^T_{t=1}\|\nabla F_{t,m}(\bw_t)\|^2~,
\end{aligned}
\end{equation}
where $F_{t,m}(\bw_t)\triangleq\frac1m\sum^{m-1}_{i=0}\ell_{t-i}(\bw_t)$, $1\le m\le T$, and $\ell_i(\cdot)=0$ for $i\le0$. The reason for using sliding-window in $F$, especially a large window, can be justified by Theorem 2.7 in \cite{hazan2017efficient}.
\section{Algorithm \& Theoretical Guarantees}
\label{sec:algo}
\subsection{Stochasticity  of Online Meta-learning Algorithms}

In practice, $\mathcal{D}_t^{ts}$ is typically just a random sample batch of the whole test-set, the losses and gradients obtained at each round are thus (unbiased) estimates of the true ones. This is the stochastic setting which we formalize by making following assumptions.

\begin{assumption}
We assume that at each round $t$, each call to any stochastic gradient oracle $\bg_i$, $i\in\{t-m+1,\ldots,t\}$, yields an i.i.d.~random vector
$\bg_i(\bw_t, \xi_{t,i})$ with the following properties:
\begin{enumerate}[label=(\alph*), topsep=1pt]
\item $\E_{\xi_{t,i}}\left[\bg_i(\bw_t,\xi_{t,i})|\xi_{1:t-1}\right]
=\nabla \ell_i(\bw_t)$~;
\item $\E_{\xi_{t,i}}\left[\left\|\bg_i(\bw_t,\xi_{t,i})-\nabla \ell_i(\bw_t)\right\|^2\vert\xi_{1:t-1}\right]
\le\sigma^2$~;
\item Mutual independence: for $i\ne j$,
\begin{align*}
&\E_{\xi_{t,i},\xi_{t,j}}[\langle\bg_i(\bw_t,\xi_{t,i}),\ \bg_j(\bw_t,\xi_{t,j})\rangle|\xi_{1:t-1}] = \\ &\langle\E_{\xi_{t,i}}[\bg_i(\bw_t,\xi_{t,i})|\xi_{1:t-1}],\ \E_{\xi_{t,j}}[\bg_j(\bw_t,\xi_{t,j})|\xi_{1:t-1}]\rangle~.
\end{align*}
\end{enumerate}
where $\xi_{1:t-1}=\{\xi_{1,1},\xi_{2,1},\xi_{2,2},\ldots,\xi_{t-1,t-m},\ldots,\xi_{t-1,t-1}\}$, and $\E_{\xi_{t,i}}[\bu|\xi_{1:t-1}]$ denotes the conditional expectation of $\bu$ with respect to $\xi_{1:t-1}$. Also note that $\bg_i(\cdot)=0$ for $i\le0$.
\label{ass:sg}
\end{assumption}

Hazan et al. proposed a time-smoothed online gradient descent algorithm~\cite{hazan2017efficient} for such case. Yet, that algorithm's performance critically relies on the choice of the step-size $\eta$, and may even diverge if $\eta>\frac2\beta$ where $\beta$ is the (often unknown) smoothness of the loss function. We thus propose to use the AdaGrad-Norm~\cite{ward2018adagrad} algorithm (Algorithm~\ref{algo:adagrad}) as the online learning algorithm $\mathcal{A}$ in Algorithm~\ref{algo:mlol} instead. Here, $b_1>0$ is the initialization of the accumulated squared norms and prevents division by 0, while $\eta > 0$ is to ensure homogeneity and that the units match.

\begin{algorithm}[t]
\begin{algorithmic}[1]
\STATE{\textbf{Input}: Initialize $\bw_1\in\R^d$, $b_1>0$, $\eta>0$.}
\FOR{$t = 1,\ldots,T$}
\STATE{Generate $\bG_{t,m}(\bw_t) = \frac{1}{m}\sum^{m-1}_{i=0}\bg_{t-i}(\bw_{t}, \xi_{t,t-i})$}
\STATE{$b_{t+1}^2\leftarrow b^2_{t} + \|\bG_{t,m}(\bw_t)\|^2$}
\STATE{$\bw_{t+1} \leftarrow \bw_{t}-\frac{\eta}{b_{t+1}}\bG_{t,m}(\bw_t)$}
\ENDFOR
\end{algorithmic}
\caption{AdaGrad-Norm}
\label{algo:adagrad}
\end{algorithm}
\subsection{Regret Analysis}
\label{subsec:analysis}

We present below an analysis of this algorithm assuming the loss function $\ell:\mathcal{K}\rightarrow\R$ satisfies:
\begin{assumption} $\ell$ is twice differentiable and $\forall \bu,\bv\in\mathcal{K}$:
\begin{enumerate}[label=(\alph*)]
\item$L$-Lipschitz: $\|\ell(\bu) - \ell(\bv)\|\le L\|\bu-\bv\|$~.
\item $\beta$-smooth: $\|\nabla \ell(\bu)-\nabla \ell(\bv)\| \leq \beta \|\bu-\bv\|$~.\\
Note that this implies~\cite[Lemma 1.2.3]{Nesterov2003}:
\vspace{-0.75em}\begin{equation}
\label{eq:smooth2}
\left|\ell(\bv)-\ell(\bu)-\langle \nabla \ell(\bu), \bv-\bu\rangle\right|
\leq \frac{\beta}{2}\|\bv-\bu\|^2~.
\end{equation}
\vspace{-1.7em}
\item$H$-Hessian-Lipschitz: $\|\nabla^2\ell(\bu) - \nabla^2\ell(\bv)\|\le H\|\bu-\bv\|$~.
\item$M$-Bounded: $|\ell(\bu)|\le M$
\end{enumerate}
\label{ass:loss}
\end{assumption}

Under Assumption~\ref{ass:loss} of $\ell$, we can derive the following properties of $\ell_t$ (the proof can be found in the Appendix):
\begin{lemma}
Assuming Assumption~\ref{ass:loss} holds, $\ell_t$ is $M$-Bounded, $L'\triangleq(1+\alpha \beta)L$-Lipschitz, and $\beta'\triangleq(\alpha LH + (1+\alpha\beta)^2\beta)$-smooth.
\label{lm:property}
\end{lemma}

The following theorem shows that by selecting $m\in\Theta(T)$, a logarithmic regret of the algorithm is guaranteed  w.r.t. any $b_1, \eta>0$.
\begin{thm}
Let $\ell_1,\ldots,\ell_T$ satisfy Assumptions~\ref{ass:loss}. Then, feeding Algorithm~\ref{algo:adagrad} into Algorithm~\ref{algo:mlol} with access to stochastic gradient oracles satisfying Assumptions~\ref{ass:sg} gives the following upper bound of $\mathcal{R}_m(T)$, with probability $1-\delta$:
\[
\mathcal{R}_m(T)
\le
\frac{48C^2}{\delta^2}
+
\frac{8b_1C}{\delta}
+
\frac{8\sigma C\sqrt{T}}{\delta^{3/2}\sqrt{m}}~,
\]
where $C =  \frac{4MT}{\eta m}
+
\left(\frac{\eta\beta'+4\sigma/\sqrt{m}}{2}\right)\ln\left(1 + \frac{2(\sigma^2/m+L'^2)T}{b_1^2}\right)$~.
\label{thm:adagrad}
\end{thm}
\vspace{-2px}
Before showing the proof of Theorem 1, we need the following technical lemmas whose proofs can be found in the Appendix. For simplicity, we denote $\E_t$ as condition on $\xi_{1:t-1}$ and take expectation w.r.t.~$\xi_{t, t-m+1},\ldots,\xi_{t, t}$:
\vspace{-2px}
\begin{lemma}
As $\bG_{t,m}(\bw_t) = \frac{1}{m}\sum^{m-1}_{i=0}\bg_{t-i}(\bw_t,\xi_{t,t-i})$, and $F_{t,m}(\bw_t)=\frac{1}{m}\sum^{m-1}_{i=0}\ell_{t-i}(\bw_t)$, Assumption~\ref{ass:sg} gives us:
\begin{enumerate}[label=(\alph*)]
\item $\E_{t}\left[\bG_{t,m}(\bw_t)\right]
=\nabla F_{t,m}(\bw_t)$\label{ass:unbiase}
\item $\E_{t}\left[\left\|\bG_{t,m}(\bx_t)-\nabla F_{t,m}(\bx_t)\right\|^2\right]
\le\frac{\sigma^2}{m}$
\end{enumerate}
\label{lm:stocg}
\end{lemma}

\vspace{-2em}
\begin{lemma} Given Assumption~\ref{ass:loss}(d), we have:
$\sum^T_{t=1}\E[F_{t,m}(\bw_t)-F_{t,m}(\bw_{t+1})]\le\frac{4MT}{m}$.
\label{lm:sumobj}
\end{lemma}
\begin{lemma} [\cite{li2019convergence}, Lemma 9] 
\label{lm:integ}
Let $h:[0,+\infty)\rightarrow [0, +\infty)$ be a nonincreasing function, and $a_i\geq0$ for $i = 0, \cdots, T$.
Then
\begin{align*}
\sum_{t=1}^T a_t h\left(a_0+\sum_{i=1}^{t} a_i\right) 
&\leq \int_{a_0}^{\sum_{t=0}^T a_t} h(x) dx~.
\end{align*}
\end{lemma}
\begin{figure*}[t]
    \centering
    \includegraphics[width=0.95\textwidth]{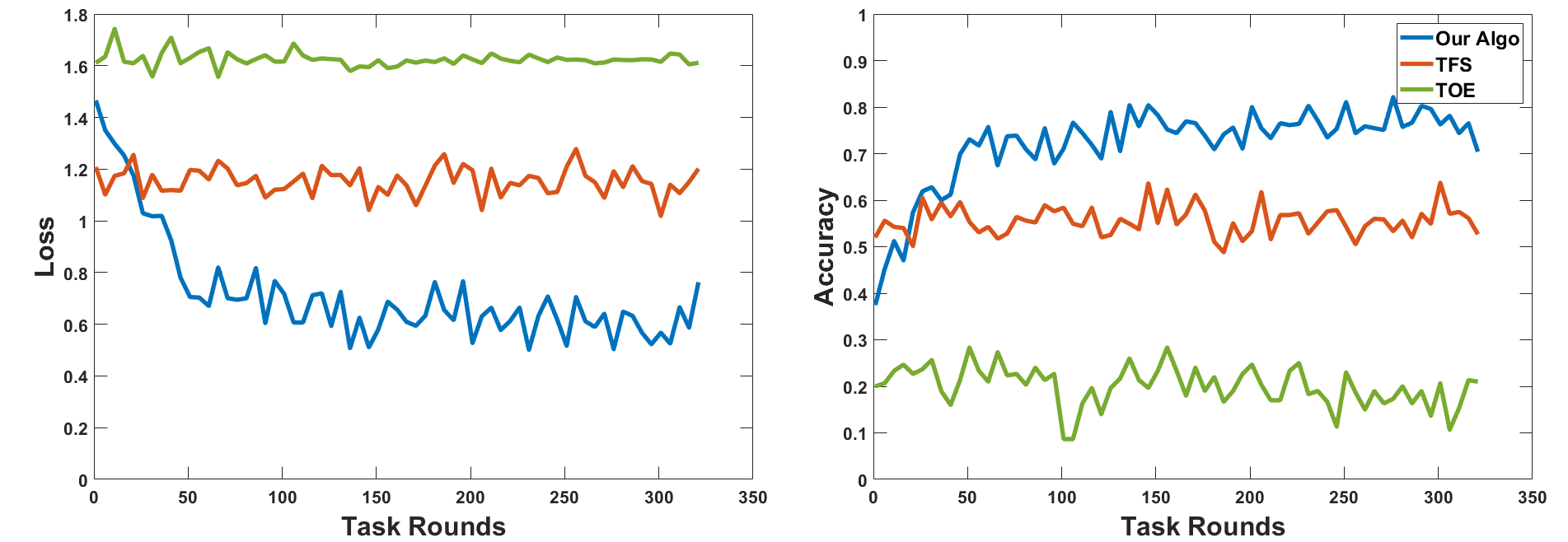}
    \vspace{-1.0em}
    \caption{The comparison between our algorithm, TOE, and TFS on performing few-shot classification on the Omniglot dataset.}
    \vspace{-1.3em}
    \label{fig:result}
\end{figure*}

\vspace{-2em}
\begin{proof}[Proof of Theorem~\ref{thm:adagrad}]
The proof follows that of Theorem 2.1 in~\cite{ward2018adagrad}.

\vspace{0.5em}First, as the average of $m$ $\beta'$-smooth functions, $F_{t,m}$ is also $\beta'$-smooth. Using the property in Assumption~\ref{ass:loss}(b) and the update formula (Line 5) in Algorithm~\ref{algo:adagrad} we have:
\begin{align*}
&\frac{F_{t,m}(\bw_{t+1})-F_{t,m}(\bw_t)}{\eta}\\
\le&
-\langle \nabla F_{t,m}(\bw_t), 
\frac{\bG_{t,m}(\bw_t)}{b_{t+1}} \rangle
+
\frac{\eta\beta'}{2b_{t+1}^2}\|\bG_{t,m}(\bw_t)\|^2~.
\end{align*}

Denote $\tilde{b}_{t+1}^2\triangleq b_t^2 + \|\nabla F_{t,m}(\bw_t)\|^2 + \frac{\sigma^2}{m}$, and take expectation w.r.t.~$\xi_{t, t-m+1},\ldots,\xi_{t, t}$ conditioned on $\xi_{1:t-1}$ (namely $\E_t[\cdot]$) :
\begin{align}
&\frac{\E_t[{F_{t,m}(\bw_{t+1})-F_{t,m}(\bw_t)}]}{\eta}
\\
\le
&\E_t\left[\left(\frac1{\tilde{b}_{t+1}}
-
\frac{1}{b_{t+1}}\right) \langle \nabla F_{t,m}(\bw_t), \bG_{t,m}(\bw_t)\rangle\right]\label{eq:inner}\\
&-\frac{\|\nabla F_{t,m}(\bw_t)\|^2}{\tilde{b}_{t+1}}
+
\frac{\eta\beta'}{2}\E_t\left[\frac{1}{b_{t+1}^2}\|\bG_{t,m}(\bw_t)\|^2\right]~.
\end{align}

Second, from the definition of $b_{t+1}$ and $\tilde{b}_{t+1}$ we have:
\begin{align*}
\left|\frac{1}{\tilde{b}_{t+1}}
-
\frac{1}{b_{t+1}}\right|
&=
\frac{\left|\|\bG_{t,m}(\bw_t)\|^2 - \|\nabla F_{t,m}(\bw_t)\|^2 - \sigma^2/m\right|}{b_{t+1}\tilde{b}_{t+1}\left(b_{t+1}+\tilde{b}_{t+1}\right)}\\[0.5em]
&\le\frac{\left|\|\bG_{t,m}(\bw_t)\| - \|\nabla F_{t,m}(\bw_t)\|\right|}{b_{t+1}\tilde{b}_{t+1}}
+
\frac{\sigma/\sqrt{m}}{b_{t+1}\tilde{b}_{t+1}}~.
\end{align*}

Using this, and Jensen's inequality on $|\cdot|$ which is a convex function, we can upper-bound Equation~\eqref{eq:inner} by its absolute value which in turn can be upper-bounded by:
\begin{align}
&{{\frac{\E_t\left[\left|\|\bG_{t,m}(\bw_t)\| - \|\nabla F_{t,m}(\bw_t)\|\right|\|\bG_{t,m}(\bw_t)\|\|\nabla F_{t,m}(\bw_t)\|\right]}{b_{t+1}\tilde{b}_{t+1}}}}\label{eq:first}\\[0.5em]
&+
\frac{\E_t\left[\|\bG_{t,m}(\bw_t)\|\|\nabla F_{t,m}(\bw_t)\|\sigma/\sqrt{m}\right]}{b_{t+1}\tilde{b}_{t+1}}\label{eq:second}~.
\end{align}

Third, by using inequality $ab\le\frac{\lambda}{2}a^2 + \frac{1}{2\lambda}b^2$ with $\lambda = \frac{2\sigma^2/m}{\tilde{b}_{t+1}}$, $a=\frac{\|\bG_{t,m}(\bw_t)\|}{b_{t+1}}$, Equation~\eqref{eq:first} can be upper bounded by:
\begin{equation*}
\frac{\|\nabla F_{t,m}(\bw_t)\|^2}{4\tilde{b}_{t+1}}
+
\frac{\sigma}{\sqrt{m}}\E_t\left[\frac{\|\bG_{t,m}(\bw_t)\|^2}{b_{t+1}^2}\right]~,
\end{equation*}
where we used that $\left|\|\bu\| - \|\bv\|\right|\le\|\bu- \bv\|$ holds for $\forall\bu,\bv\in\R^d$. 

\vspace{0.5em}
Applying $ab\le\frac{\lambda}{2}a^2 + \frac{1}{2\lambda}b^2$ again but with $\lambda = \frac{2}{\tilde{b}_{t+1}}$, $a=\frac{\|\bG_{t,m}(\bw_t)\|\sigma/\sqrt{m}}{b_{t+1}}$, we can upper bound eq.~\eqref{eq:second} by:
\begin{equation*}
\frac{\|\nabla F_{t,m}(\bw_t)\|^2}{4\tilde{b}_{t+1}}
+
\frac{\sigma}{\sqrt{m}}\E_t\left[\frac{\|\bG_{t,m}(\bw_t)\|^2}{b_{t+1}^2}\right]~.
\end{equation*}

Fourth, putting above two inequalities back, and then in turn putting the result back into Equation~\eqref{eq:inner} give us:
\begin{align*}
&\frac{\E_t[{F_{t,m}(\bw_{t+1})]-F_{t,m}(\bw_t)}}{\eta}\\
\le
&-\frac{\|\nabla F_{t,m}(\bw_t)\|^2}{2\tilde{b}_{t+1}}
+
\left(\frac{\eta\beta'}{2}+\frac{2\sigma}{\sqrt{m}}\right)\E_t\left[\frac{1}{b_{t+1}^2}\|\bG_{t,m}(\bw_t)\|^2\right]~.
\end{align*}

\vspace{-2px}
Rearrange terms, then for both sides, take expectation w.r.t.~$\xi_{1:t-1}$ and sum from $t=1$ to $T$ :
\begin{align}\notag
&\sum^{T}_{t=1}\E\left[\frac{\|\nabla F_{t,m}(\bw_t)\|^2}{2\tilde{b}_{t+1}}\right]
\\
\le
&\frac{\sum^{T}_{t=1}[\E[F_{t,m}(\bw_t)]-\E[F_{t,m}(\bw_{t+1})]]}{\eta}\label{eq:term1}\\
&+
\left(\frac{\eta\beta'+4\sigma/\sqrt{m}}{2}\right)\E\sum^{T}_{t=1}\frac{\|\bG_{t,m}(\bw_t)\|^2}{b_{t+1}^2}\label{eq:term2}~.
\end{align}

As $b_{t+1}^2 = b_1^2 + \sum^{t}_{i=1}\|\bG_{i,m}(\bw_i)\|^2$, letting $h(x)$ be $1/x$ in Lemma~\ref{lm:integ} gives us:
\begin{align*}
\E\left[\sum^{T}_{t=1}\frac{\|\bG_{t,m}(\bw_t)\|^2}{b_{t+1}^2}\right]
\le
\ln\left(1 + \frac{\sum^{T}_{t=1}\E\left[\|\bG_{t,m}(\bw_t)\|^2\right]}{b_1^2}\right)~,
\end{align*}
where we used Jensen's inequality for $\ln(x)$ which is a concave function in $(0,+\infty)$.

Since each $\ell_t$ is $L'$-Lipschitz, so is $F_{t,m}(\cdot)$, thus, using Cauchy-Schwartz inequality:
\begin{align}
\E\left[\|\bG_{t,m}(\bw_t)\|^2\right]
&\le
2\E\left[\|\bG_{t,m}(\bw_t) - \nabla F_{t,m}(\bw_t)\|^2\right]\nonumber\\
&\phantom{\le}+ 2\E\left[\|\nabla F_{t,m}(\bw_t)\|^2\right]\label{eq:grad}\\
&\le
2(\sigma^2/m+L'^2)\nonumber~.
\end{align}

Putting the above inequality back into Equation~\eqref{eq:term2} and Lemma~\ref{lm:sumobj} back into Equation~\eqref{eq:term1}, we have:
\begin{align}\notag
\sum^{T}_{t=1}\E\left[\frac{\|\nabla F_{t,m}(\bw_t)\|^2}{2\tilde{b}_{t+1}}\right]
\le&
\frac{4MT}{\eta m}\\
+
\left(\frac{\eta\beta'+4\sigma/\sqrt{m}}{2}\right)&\ln\left(1 + \frac{2(\sigma^2/m+L'^2)T}{b_1^2}\right)\label{eq:cons}~.
\end{align}

\vspace{-2px}
Finally, using Markov's inequality, with probability $1-\delta_1$, Lemma~\ref{lm:stocg}(b) gives us:
\begin{equation*}
\sum^{T}_{t=1}\|\nabla F_{t,m}(\bw_t) -\bG_{t,m}(\bw_t)\|^2\le\frac{T\sigma^2}{m\delta_1}~.
\end{equation*}

\vspace{-2px}
Denote $Z\triangleq\sum^{T}_{t=1}\|\nabla F_{t,m}(\bw_t)\|^2$. Using similar derivation in Equation~\eqref{eq:grad}, with probability $1-\delta_1$ we have:
\begin{align*}
b_{T}^2 + \|\nabla F_{T,m}(\bw_{T})\|^2 + \sigma^2/m
\le\ 
b_1^2 + 2Z + 2T\frac{\sigma^2}{m\delta_1}
\end{align*}

\vspace{-2px}
This means, with probability $1-\delta_1$, we have:
\begin{align*}
\sum^{T}_{t=1}\frac{\|\nabla F_{t,m}(\bw_t)\|^2}{2\tilde{b}_{t+1}}
&\ge
\frac{\sum^{T}_{t=1}\|\nabla F_{t,m}(\bw_t)\|^2}{2\sqrt{b_{T}^2 + \|\nabla F_{T,m}(\bw_T)\|^2 + Z +\sigma^2/m}}\\
&\ge
\frac{\sum^{T}_{t=1}\|\nabla F_{t,m}(\bw_t)\|^2}{2\sqrt{b_1^2 + 3Z + 2T\frac{\sigma^2}{m\delta_1}}}~.
\end{align*}

\vspace{-2px}
Denote the right-hand side of Equation~\eqref{eq:cons} as $C$, and use Markov's inequality again we have, with probability $1-\delta_2$:
\begin{equation*}
\sum^{T}_{t=1}\frac{\|\nabla F_{t,m}(\bw_t)\|^2}{2\tilde{b}_{t+1}}
\le
\frac{C}{\delta_2}~.
\end{equation*}

\vspace{-2px}
Therefore, with probability $1-\delta_1-\delta_2$, we have
\begin{equation*}
\frac{Z}{2\sqrt{b_1^2 + 3Z + 2T\frac{\sigma^2}{m\delta_1}}}
\le
\frac{C}{\delta_2}~.
\end{equation*}

By solving the above "quadratic" inequality of $Z$ and letting $\delta_1=\delta_2=\frac{\delta}{2}$, we arrive at the end.
\end{proof}

\section{Experiment}
\label{sec:exp}

We evaluated our algorithm on the few-shot image classification task of the Omniglot~\cite{lake2015human} dataset which consists of 20 instances of 1623 characters from 50 different alphabets. The dataset is augmented with rotations by multiples of 90 degrees following~\cite{santoro2016meta}.

We employed the $N$-way $K$-shot protocol~\cite{vinyals2016matching}: at each round, pick $N$ unseen characters irrespective of alphabets. Provide the meta-learner $\bw_t$ with $K$ different drawings of each of the $N$ characters as the training set $\mathcal{D}^{tr}$, then evaluate the adapted model $\hat{\bw}_t$’s ability on new unseen instances within the $N$ classes (namely the test set $\mathcal{D}^{ts}$). We chose the 5-way 5-shot scheme, and used 15 samples per character for testing following~\cite{ravi2016optimization}.

The model we used is a CNN following~\cite{vinyals2016matching}. It contains 4 modules, each of which is a 3$\times$3 convolution with 64 filters followed by batch normalization~\cite{ioffe2015batch}, a ReLu non-linearity and 2$\times$2 max-pooling. Images are downsampled to 28$\times$28 so that the resulting feature map of the last hidden layer is 1$\times$1$\times$64. The last layer is fed into a fully connected layer and the loss we used is the Cross-Entropy loss.

To study if our algorithm provides any empirical benefit over traditional methods, we compare it to two benchmark algorithms~\cite{finn2019online}: Train on Everything (TOE), and Train from Scratch (TFS). On each round $t$, both initialize a new model. The difference is that TOE trains over all available data, both training and testing, from all past tasks, plus $D_t^{tr}$ at current round, while TFS only uses $D_t^{tr}$ for training.

The experiments are performed in PyTorch~\cite{paszke2017automatic}, and parameters are by default if no specification is provided. For the parameter $\alpha$ in the local adapter strategy $\bU(\cdot)$ in Algorithm~\ref{algo:mlol}, we set it to be 0.1 everywhere, and the gradient descent step is performed only once for each task. For the AdaGrad-Norm algorithm (Algorithm~\ref{algo:adagrad}) we used, we set $b_1=\eta=1$ as suggested in the original paper~\cite{ward2018adagrad}. The TFS and TOE used Adam~\cite{Kingma2014AdamAM} with default parameters.

The result is shown in Figure~\ref{fig:result} which suggests that our algorithm gradually accumulates prior knowledge, which enables fast learning of later tasks. TFS provides a good example of how CNN performs when the training data is scarse. On the contrary, TOE behaves nearly as random guessing. The inferiority of TOE to TFS is somehow surprising, as TOE has much more training data than TFS. The reason is that TOE regards all training data as coming from a single distribution, and tries to learn a model that works for all tasks. Thus, when tasks are substantially different from each other, TOE might even incur negative transfer and fail to solve any single task as has been observed in~\cite{parisotto2015actor}. Meanwhile, by using training data of the current task only, TFS avoids negative transfer, but also rules out learning of any connection between tasks. Our algorithm, in contrast, is designed to discover common structures across tasks, and use these information to guide fast adaptation to new tasks.
\section{Conclusion}
\label{sec:conc}
The continual lifelong learning problem is common in real-life, where an agent needs to accumulate knowledge from every task it encounters, and utilize that knowledge for fast learning of new tasks. To solve this problem, we can combine the meta-learning and the online-learning paradigms to form the online meta-learning framework. In this work, we generalized this framework to the non-convex setting, and introduced the local regret to replace the original regret definition. We applied it to the stochastic setting, and showed its superiority both in theory and practice. In the future work, we would like to evaluate our algorithm on harder learning problems over larger scale datasets.




\clearpage \newpage
\vfill\pagebreak
\clearpage
\newpage
\bibliographystyle{IEEEbib}
\bibliography{refs}

\newpage
\onecolumn
\appendix
\section{Appendix}
\setcounter{equation}{0}
\subsection{Proof of Lemma~\ref{lm:property}}
\textbf{Lemma 1. }\emph{
Assuming Assumption~\ref{ass:loss}, $\ell_t$ is $M$-Bounded, $L'\triangleq(1+\alpha \beta)L$-Lipschitz, and $\beta'\triangleq(\alpha LH + (1+\alpha\beta)^2\beta)$-smooth.}

\begin{proof}
We first write out the complete formula of $\ell_t$:
\begin{align*}
\ell_t(\bw)
&=
\ell(\hat{\bw}, \mathcal{D}^{ts}_t)\\
&=
\E_{\bx^{ts},y^{ts}\sim\mathcal{D}^{ts}_t}[\ell(\bU(\bw, \mathcal{D}_t^{tr}), \bx^{ts}; y^{ts})]\\
&=
\E_{\bx^{ts},y^{ts}\sim\mathcal{D}^{ts}_t}[\ell(\bw - \alpha\nabla \E_{\bx^{tr},y^{tr}\sim\mathcal{D}^{tr}_t}[\ell(\bw, \bx^{tr};y^{tr})], \bx^{ts}; y^{ts})]\\
&\triangleq
f_t(\bw-\alpha\nabla\hat{f}_t(\bw))~.
\end{align*}

The $M$-Boundedness is straight-forward.

To show the Lipschitzness, we derive $\nabla\ell_t$:
\begin{equation*}
\nabla\ell_t(\bw)
=
(\bI-\alpha\nabla^2\hat{f}_t(\bw))\nabla f_t(\bw-\alpha\nabla\hat{f}_t(\bw))~.
\end{equation*}

Note that $f_t$ and $\hat{f}_t$ both share the properties of $\ell$, thus, from Assumption~\ref{ass:loss}(a,b), we have:
\begin{equation*}
\|\nabla\ell_t(\bw)\|
\le
(1+\alpha\beta)\|\nabla f_t(\bw-\alpha\nabla\hat{f}_t(\bw))\|
\le
(1+\alpha\beta)L~.
\end{equation*}

Next, denoting $\bU(\bw,\mathcal{D}^{tr}_t)$ as $\bU_t(\bw)$, we have $\forall \bu,\bv\in\mathcal{K}$:
\begin{align*}
&\phantom{==}\|\nabla\ell_t(\bu)-\nabla\ell_t(\bv)\|\\
&=
\|\nabla\bU_t(\bu)\nabla f_t(\bU_t(\bu)) - \nabla\bU_t(\bv)\nabla f_t(\bU_t(\bv))\|\\
&=
\|\nabla\bU_t(\bu)\nabla f_t(\bU_t(\bu)) - \nabla\bU_t(\bv)\nabla f_t(\bU_t(\bu))+ \nabla\bU_t(\bv)\nabla f_t(\bU_t(\bu)) - \nabla\bU_t(\bv)\nabla f_t(\bU_t(\bv))\|\\
&\le
\|(\nabla\bU_t(\bu) - \nabla\bU_t(\bv))\nabla f_t(\bU_t(\bu))\|
+ \|\nabla\bU_t(\bv)(\nabla f_t(\bU_t(\bu)) - \nabla f_t(\bU_t(\bv)))\|\\
&=
\alpha\|(\nabla^2\hat{f}_t(\bu) - \nabla^2\hat{f}_t(\bv))\nabla f_t(\bU_t(\bu))\|
+
\|(\bI-\alpha\nabla^2\hat{f}_t(\bv))(\nabla f_t(\bU_t(\bu)) - \nabla f_t(\bU_t(\bv)))\|\\
&\le
\alpha LH\|\bu-\bv\|
+
(1+\alpha\beta)\|\nabla f_t(\bU_t(\bu)) - \nabla f_t(\bU_t(\bv))\|\\
&\le
\alpha LH\|\bu-\bv\|
+
(1+\alpha\beta)\beta\|\bU_t(\bu) - \bU_t(\bv)\|\\
&\le
\alpha LH\|\bu-\bv\|
+
(1+\alpha\beta)^2\beta\|\bu - \bv\|~,
\end{align*}
where the first inequality uses the triangle inequality of a norm; the second inequality uses the smoothness and hessian-Lipschitzness assumptions; the third inequality uses the smoothness assumption.

We are left to prove the last inequality:
\begin{align*}
\|\bU_t(\bu)-\bU_t(\bv)\|
&=
\|\bu-\alpha\nabla\hat{f}_t(\bu) - \bv+\alpha\nabla\hat{f}_t(\bv)\|\\
&=
\|\bu- \bv
-
\alpha(\nabla\hat{f}_t(\bu)-\nabla\hat{f}_t(\bv))\|\\
&\le
\|\bu- \bv\|
+
\alpha\|\nabla\hat{f}_t(\bu)-\nabla\hat{f}_t(\bv)\|\\
&\le
(1+\alpha\beta)\|\bu-\bv\|~,
\end{align*}
where the the first inequality uses the triangle inequality of a norm, and the second inequality uses the smoothness assumption.
\end{proof}

\subsection{Proof of Lemma~\ref{lm:stocg}}
\textbf{Lemma 2. } \emph{
As $\bG_{t,m}(\bw_t) = \frac{1}{m}\sum^{m-1}_{i=0}\bg_{t-i}(\bw_t,\xi_{t,t-i})$, and $F_{t,m}(\bw_t)=\frac{1}{m}\sum^{m-1}_{i=0}\ell_{t-i}(\bw_t)$, Assumption~\ref{ass:sg} gives us:
\begin{enumerate}[label=(\alph*)]
\item $\E_{t}\left[\bG_{t,m}(\bw_t)\right]
=\nabla F_{t,m}(\bw_t)$\label{ass:unbiase}
\item $\E_{t}\left[\left\|\bG_{t,m}(\bx_t)-\nabla F_{t,m}(\bx_t)\right\|^2\right]
\le\frac{\sigma^2}{m}$
\end{enumerate}}

\begin{proof}
Note that $\E_t$ denotes conditioning on $\xi_{1:t-1}$ and take expectation w.r.t.~$\xi_{t, t-m+1},\ldots,\xi_{t, t}$.

In Assumption~\ref{ass:sg}(a) we assume $\E_{\xi_{t,i}}\left[\bg_i(\bw_t,\xi_{t,i})|\xi_{1:t-1}\right]
=\nabla \ell_i(\bw_t)$ for $i\in\{t-m+1,\ldots,t\}$, the linearity of expectation immediately gives us $\E_{t}\left[\bG_{t,m}(\bw_t)\right]
=\nabla F_{t,m}(\bw_t)$. 

To see the second part, we only need to expand $\E_t\left[\left\|\bG_{t,m}(\bx_t)-\nabla F_{t,m}(\bx_t)\right\|^2\right]$ as:
\begin{align*}
&\frac1{m^2}\E_t\left[\left\|\sum^{m-1}_{i=0}\bg_{t-i}(\bw_t,\xi_{t,t-i}) - \nabla\ell_{t-i}(\bw_t)\right\|^2\right]\\
=
&\frac1{m^2}\sum^{m-1}_{i=0}\sum^{m-1}_{j=0}\E_t\left[\langle\bg_{t-i}(\bw_t,\xi_{t,t-i}) - \nabla\ell_{t-i}(\bw_t),\  \bg_{t-j}(\bw_t,\xi_{t,t-j}) - \nabla\ell_{t-j}(\bw_t)\rangle\right]\\
=
&\frac1{m^2}\sum^{m-1}_{i=0}\E_t\left[\|\bg_{t-i}(\bw_t,\xi_{t,t-i}) - \nabla\ell_{t-i}(\bw_t)\|^2\right]
+
\frac1{m^2}\sum^{m-1}_{i=0}\sum_{j\ne i}\E_t\left[\langle\bg_{t-i}(\bw_t,\xi_{t,t-i}) - \nabla\ell_{t-i}(\bw_t),\  \bg_{t-j}(\bw_t,\xi_{t,t-j}) - \nabla\ell_{t-j}(\bw_t)\rangle\right]~.
\end{align*}

Each item of the first part in the last equation can be bounded by $\sigma^2$ according to Assumption~\ref{ass:sg}(b), which leads to a $\frac{\sigma^2}m$ overall upper-bound.

For the second part, we need to use the Mutual Independence assumption (namely Assumption~\ref{ass:sg}(c)):
\begin{align*}
&\E_t\left[\langle\bg_{t-i}(\bw_t,\xi_{t,t-i}) - \nabla\ell_{t-i}(\bw_t),\  \bg_{t-j}(\bw_t,\xi_{t,t-j}) - \nabla\ell_{t-j}(\bw_t)\rangle\right]\\
=
&\langle\E_t\left[\bg_{t-i}(\bw_t,\xi_{t,t-i}) - \nabla\ell_{t-i}(\bw_t)\right],\  \E_t\left[\bg_{t-j}(\bw_t,\xi_{t,t-j}) - \nabla\ell_{t-j}(\bw_t)\rangle\right]~.
\end{align*}

Use Assumption~\ref{ass:sg}(a) again we know that the above equation equals to 0. This proves part (b) of this lemma.
\end{proof}

\subsection{Proof of Lemma~\ref{lm:sumobj}}
\textbf{Lemma 3. } \emph{Given Assumption~\ref{ass:loss}(d), we have:
$\sum^T_{t=1}\E[F_{t,m}(\bw_t)-F_{t,m}(\bw_{t+1})]\le\frac{4MT}{m}$.}
\begin{proof}
\begin{align*}
&\sum^T_{t=1}\E[F_{t,m}(\bw_t)-F_{t,m}(\bw_{t+1})]\\
=
&\sum^T_{t=2}\E[F_{t,m}(\bw_t)-F_{t-1,m}(\bw_{t})]
+
F_{1,m}(\bw_1)
-
\E[F_{T,m}(\bw_{T+1})]\\
=
&\sum^T_{t=2}\frac1m\sum^{m-1}_{i=0}\E[\ell_{t-i}(\bw_t)-\ell_{t-1-i}(\bw_{t})]
+
\ell_{1}(\bw_1)
-
\frac1m\sum^{m-1}_{i=0}\E[\ell_{T-i}(\bw_{T+1})]\\
=
&\sum^T_{t=2}\frac1m\E[\ell_{t}(\bw_t)-\ell_{t-m}(\bw_{t})]
+
\ell_{1}(\bw_1)
-
\frac1m\sum^{m-1}_{i=0}\E[\ell_{T-i}(\bw_{T+1})]\\
\le
&\frac{2MT}{m} + M + M
\le
\frac{4MT}{m}~,
\end{align*}
where we use the definition that $\ell_i(\cdot)=0$ for $i\le0$, and $1\le m\le T$.
\end{proof}

\subsection{Proof of Lemma~\ref{lm:integ}}
\textbf{Lemma 4. } [\cite{li2019convergence}, Lemma 9] 
\emph{Let $h:[0,+\infty)\rightarrow [0, +\infty)$ be a nonincreasing function, and $a_i\geq0$ for $i = 0, \cdots, T$.
Then
\begin{align*}
\sum_{t=1}^T a_t h\left(a_0+\sum_{i=1}^{t} a_i\right) 
&\leq \int_{a_0}^{\sum_{t=0}^T a_t} h(x) dx~.
\end{align*}}
\begin{proof}
Denote $s_t=\sum_{i=0}^{t} a_i$.
\begin{align*}
a_t h(s_t) 
=  \int_{s_{t-1}}^{s_t} h(s_t) d x 
\leq \int_{s_{t-1}}^{s_t} h(x) d x~.
\end{align*}
Summing over $t=1, \cdots, T$, we have the stated bound.
\end{proof}

\end{document}